\documentclass{fundam}

\usepackage{adjustbox}
\usepackage{amsmath}
\usepackage{amsmath,amssymb,amsfonts}           \usepackage[T1]{fontenc}
\usepackage{newtxmath,newtxtext}                
\usepackage{rotating}
\usepackage{pdflscape}
\usepackage{tabulary}                                                           
\usepackage{amsmath}
\usepackage{array}
\usepackage{booktabs}  

\publyear{22}
\papernumber{2102}
\volume{185}
\issue{1}

\usepackage{url} 
\usepackage[ruled,lined]{algorithm2e}
\usepackage{graphicx}
\usepackage{tikz}
\usetikzlibrary{automata, positioning, arrows}

\begin{document}

\title{Improving Heuristic-based Process Discovery Methods by Detecting Optimal Dependency Graphs}

\address{ Iran University of Science and Technology, Tehran 16844, Iran}

\author{Maryam Tavakoli-Zanian,  Mohammad Reza Gholamiani\\
School of Industrial Engineering \\
 Iran University of Science and Technology \\Tehran 16844 Iran\\
maryam.tavakolii@gmail.com, Gholamian@iust.ac.ir

\and S. Alireza Hashemi Golpayegani\\
Computer Engineering and IT Department\\
 Amirkabir University of Technology \\
Tehran, Iran } 

\maketitle

\runninghead{M. Tavakoli-Zaniani et al.}{Improving Heuristic-based Process Discovery Methods by Detecting ...}

\begin{abstract}
Heuristic-based methods are among the most popular methods in the process discovery area. This category of methods is composed of two main steps: 1) discovering a dependency graph 2) determining the split/join patterns of the dependency graph. The current dependency graph discovery techniques of heuristic-based methods select the initial set of graph arcs according to dependency measures and then modify the set regarding some criteria. This can lead to selecting the non-optimal set of arcs. Also, the modifications can result in modeling rare behaviors and, consequently, low precision and non-simple process models. Thus, constructing dependency graphs through selecting the optimal set of arcs has a high potential for improving graphs quality. Hence, this paper proposes a new integer linear programming model that determines the optimal set of graph arcs regarding dependency measures. Simultaneously, the proposed method can eliminate some other issues that the existing methods cannot handle completely; i.e., even in the presence of loops, it guarantees that all tasks are on a path from the initial to the final tasks. This approach also allows utilizing domain knowledge by introducing appropriate constraints, which can be a practical advantage in real-world problems. To assess the results, we modified two existing methods of evaluating process models to make them capable of measuring the quality of dependency graphs. According to assessments, the outputs of the proposed method are superior to the outputs of the most prominent dependency graph discovery methods in terms of fitness, precision, and especially simplicity. 
\end{abstract}


\section{Introduction}\label{sec1}


Heuristic-based process discovery is one of the most popular approaches to process discovery because of benefiting numerous advantages, such as the ability to identify less-structured process models and handle noise and incompleteness \cite{Ref1,Ref2}. The first algorithm in this category is Heuristics Miner\cite{Ref3}, and numerous extensions have been introduced for this algorithm so far \cite{Ref4, Ref5, Ref6, Ref7}. All algorithms in this category (also called “heuristic” mining methods) are composed of two major steps. First, the dependency graph is created, including basic causal relations of tasks (i.e., prerequisite and post-requisite relations). In the second step, specific patterns in the frequency of dependencies are employed to decide on split/join types.

In the dependency graph discovery step of heuristic mining methods, the initial set of dependency graph arcs is selected according to some user-defined thresholds for minimum dependency measures. Then, some arcs are added to the initial set to modify the dependency graph with respect to some criteria, such as making it connected. This can lead to selecting a non-optimal set of arcs. Moreover, the modifications can increase the number of graph arcs, which may result in modeling infrequent behaviors (i.e., making the results less precise) and increasing the size of the process model (i.e., making the results less simple). 

Moreover, despite applying various solutions, when the dependency graphs extracted by existing heuristic-based miners contain loops, still some tasks are likely to not be on a path from the initial task to the final task. This can lead to non-sound process models.

The current heuristic mining methods also face severe limitations on the types of knowledge they can employ. While, usually, there is some precious domain knowledge among various stakeholders involved in the process, and using it can improve the outputs \cite{Ref10}.

Mathematical programming appears to be an appropriate approach to dependency graph discovery. Because, through defining an appropriate objective function, the optimal graph arcs can be identified, while, using appropriate constraints, the other problems mentioned above can be appropriately handled in the procedure of dependency graph discovery. Especially in the cases of linear programming (LP) and integer linear programming (ILP), there are efficient and powerful methods for achieving global optimal solutions. Therefore, this study, for the first time, introduces an ILP model for dependency graph discovery which 1) selects the optimal set of dependency graph arcs concerning dependency measures. 2) can ensure that all tasks of the resultant graph are on a path from the initial task to the final task (even when the output graph contains loops. 3) applying many types of domain knowledge and user-desired flexibility is easily possible by simply defining some appropriate constraints. This can be an advantage with high applicability in real-world problems \cite{Ref10}.
 
Moreover, to the best of our knowledge, there is no method in the literature for evaluating the quality of dependency graphs. Thus, in this research, two existing methods in assessing the quality of process models are modified to measure the quality of dependency graphs. The first proposed measure assesses the degree to which the sequences of events observed in the event log can be allowed according to the dependency graph pre-requisite/post-requisite relations. The second measure assesses the extent to which the sequences of events not observed in the event log are not allowed according to the dependency graph pre-requisite/post-requisite relations.
 
The rest of this paper consists of the following sections. Section~\ref{sec2}  provides a literature review and presents preliminary concepts. In Section~\ref{sec3}, the problem structure is addressed, and the proposed objective function and constraints are introduced. Section~\ref{sec4} presents the measures applied to evaluate the quality of dependency graphs. Empirical results of assessing the proposed model are then presented in Section~\ref{sec5}. The paper is summarized with the research conclusion and future studies in Section~\ref{sec6}.

\section{Literature Review}\label{sec2}
\subsection{Review of Literature and Related Studies}\label{sec2.1}
Dependency graph discovery is an essential step in heuristic-based process discovery algorithms. The first method in this category is Heuristics Miner\cite{Ref3}, which is among the most used and customized process mining algorithms \cite{Ref1}. In this algorithm, the dependency graph is constructed according to some minimum thresholds for measures of dependency between tasks. The dependency measures are calculated according to the count of tasks and direct succession relations (more details on direct succession relations can be found in Section~\ref{sec2.3}). Then, to consider long-distance dependencies, some arcs are added to the dependency graph. 

By adding some extra arcs, Heuristics Miner also ensures that in the extracted dependency graph, all tasks except for the final task have at least one output arc, and all tasks except for the initial task have at least one input arc. Using this solution, it tried to make all tasks on a path from the initial task to the final task; however, the utilized solution cannot guarantee this if the extracted dependency graph contains loops. In addition, this procedure of selecting the initial set of arcs and then adding some extra arcs to the set is likely to lead to a non-optimal set of graph arcs. Moreover, increasing the number of the arcs can result in modeling rare behaviors and increasing the dependency graph size (i.e., this can decrease the precision and simplicity of the results).

In 2012, a version of the Heuristics Miner was developed that was able to deal with the streaming event data \cite{Ref4}. This algorithm extracts the dependency graph similarly to Heuristics Miner; however, it calculates the count of direct succession and tasks differently. Another version named Heuristics Miner ++  was developed by \cite{Ref5} in 2015. This algorithm considers tasks along with their time intervals. It uses a new way of counting direct succession relations and introduced new dependency measures. 

A modified version of Heuristics Miner named “Flexible Heuristics Miner” was proposed in \cite{Ref6}, whose output model is a causal net. This algorithm constructs the dependency graph the same as the Heuristics Miner; however, it utilizes a different solution for considering long-distance dependencies. In 2017, a version of the Heuristics Miner algorithm, namely Fodina, was developed focusing on robustness against noise and flexibility \cite{Ref7}. This method introduced a modified version of dependency measures used by Heuristics Miner. Fodina can discover duplicate tasks and provides some flexible configuration options to guide the dependency graph discovery procedure with respect to the settings made by the end-user. This method proposed a new method to guarantee dependency graph connectivity. Nonetheless, the method is practically proven to be still likely to produce unconnected graphs.

To handle the drawbacks of the dependency graph discovery step of the methods mentioned above, in this study, using mathematical programming is suggested. The use of mathematical programming in process discovery is not a new subject, and numerous studies have employed this approach for process discovery \cite{Ref9,Ref10,Ref11,Ref12,Ref13,Ref14,Ref15,Ref16}. Nevertheless, \cite{Ref10} is the only research we have found that uses mathematical programming for dependency graph discovery. It proposed a process discovery method called Proximity Miner that focused on involving experts' knowledge in discovering an actionable process model. In this regard, it introduced an ILP model in dependency graph discovery that utilizes the knowledge of domain experts. After introducing a measure called proximity score, the mentioned study has defined the objective function of its proposed ILP model based on proximity score. The ILP model also utilizes three types of constraints: 1) constraints related to applying domain knowledge 2) constraints which guarantee that the initial task has no input arc and the final task has no output arc 3) constraints which guarantee that all tasks except for the initial task has at least one input arc, and all tasks except for the final task have at least one output arc. The third type of constraints uses an approach similar to Heuristics Miner to avoid extracting dependency graphs containing at least one task, not on a path from the initial task to the final task. However, as mentioned before, theoretically, when the output dependency graph contains loops, this solution cannot completely guarantee the absence of the problem.

In addition, the focus of Proximity Miner is defined on involving domain knowledge in the mining procedure. Hence, in the case of using no prior expert knowledge, the ILP model is prone to create complex dependency graphs since it tends to add an arc to the dependency graph for every direct succession which is seen between tasks in the event log. In other words, in the case of utilizing no prior knowledge, the model does not consider that observing a direct succession in the event log is not necessarily an indicator of a dependency and can be due to concurrent tasks or noises. It also does not provide a solution for distinguishing length-two loops from the other structures with similar direct succession footprints.

It also should be mentioned that there are studies in the literature that are dedicated to filtering directly-follows graphs \cite{Ref17,Ref18,Ref19} or employing mathematical programming to make anomaly-free and filtered directly-follows graphs \cite {Ref20} (which is also called log automatons in some of these studies). However, although directly-follow graphs are similar to dependency graphs in some aspects, they have fundamental and significant differences. In dependency graphs, the arcs are an indicator of causal relations between tasks, and there is no arc between tasks that are supposed to be concurrent; whereas, in directly-follows graphs, the arcs are an indicator of direct successions that occurred in the event log and there are arcs between tasks that are considered to be concurrent. Thus, the outputs of these studies are not suitable for applying the second step of heuristic-based process discovery methods, and as a consequence, they cannot be used as a step of the heuristic-based methods.

Our study is also not related to the methods that discover process models containing only XOR splits/joins (such as \cite{Ref9,Ref14,Ref15}) and the commercial software that extract process maps (such as Disco). Because the second step of the heuristic mining methods is not appliable to the outputs of these methods, and as a consequence, these methods also cannot be employed as the dependency graph discovery step of the heuristic-based methods.

\subsection{Definitions}\label{sec2.2}

\begin{definition}[Log]
An event log $L$  is a multiset of traces, each mapped onto one case. Assuming $T_L$  as the set of all traces that are present in $L$, a trace $t \in T_L$  is a finite sequence of events occurring for a process instance, which is defined as follows:

\textit (Event): $E$ is a set of events and  $E=A  \times Y  \times T$ where $A$ is a set of tasks, $T$ is a set of timestamps, and $Y$ is a finite set of attributes (case ID, event type, duration, and so on). $A_L$  is a finite set of all tasks present in $L$, and $\lvert A_L \rvert$ is the size of the set $A_L$ .
\end{definition}

\begin{definition}[Dependency Graph]
A dependency graph $DG$ is a multiset like $DG=(A_{DG},D)$  , in which $A_{DG}$  is a limited set of tasks present in $DG$ and $D=\{(a_1,a_2) \in A_{DG} \times A_{DG}\}$  is the set of directed arcs in the graph representing the dependencies between tasks. The dependency graph should have an initial task ($a_s \in A_{DG}$ ) and a final task ( $a_e \in A_{DG}$), for them $  \not \exists a_i \in A_{DG}:(a_i,a_s ) \in D$ and  $ \not \exists a_j \in A_{DG}:(a_e,a_j ) \in D$. All the tasks in$ DG$ should be placed on a path from $a_s$  to $a_e$. 

\end{definition}
\begin{definition}[short/length-one loop]
Assuming the dependency graph $DG=(A_{DG},D)$  , task $a_i \in A_{DG}$  is involved in a short loop if $(a_i,a_i) \in D$.

\end{definition}

\begin{definition}[Length-Two Loop]
Assuming the dependency graph $DG=(A_{DG},D)$  ,tasks $a_i, a_j \in a_{DG}$  are involved in a length-two loop if $(a_i,a_j) \in D \wedge (a_j,a_i) \in D$ .

\end{definition}

\subsection{Dependency Graph Discovery in Heuristic-based Process Mining Methods}\label{sec2.3}
This section addresses the main core of constructing a dependency graph in the heuristic-based process discovery methods. 

The first step is to extract the basic relations of tasks from the event log. Based on the resultant relations, the dependency graph is then created. Assuming L as an event log and tasks $a,b \in A_L$, the most important information of these algorithms extracted from the event log will be as follows:

\begin{itemize}
\item $ \mid a \mid$  is the number of times at which $a$ occurs in the event log (frequency of $a$).
\item $ \mid a>b \mid$ is the number of direct successions of $a$ by $b$, showing the number of times at which $a$ has directly (immediately) been followed by $b$ in the event log.
\item $ \mid a>>b \mid$ is the number of times in the event log at which repetition has occurred between $a$ and $b$ (i.e., $b$ occurred immediately after $a$, and then  immediately $a$ occurred again).
\end{itemize}

More details on achieving the values mentioned above can be found in \cite {Ref3}. After calculating these values, the dependency graph is created through dependency measures describing the basic causal relations (i.e., follows and precedes). Various measures have been proposed in the literature to assess the dependency of tasks/loops; however, they are mainly extracted from the aforementioned relations. When a dependency measure for two tasks exceeds the threshold set by the user, an arc will be added between those two tasks in the dependency graph. Also, a measure of length-two loop dependency is employed to extract loops of this type. When this measure exceeds the threshold set by the user, two arcs are added to the graph (one from $a$ to $b$ and one from $b$ to $a$). In the next step, the graph is modified based on some criteria, such as making the graph connected, considering long-distance dependencies, or some other flexible options.

\section{Proposed ILP Model}\label{sec3}
Constructing a dependency graph is an essential and basic step in heuristic process discovery algorithms. As mentioned in the previous section, these methods, according to some minimum dependency thresholds, select an initial set of arcs and then, based on some criteria, and user-defined flexible configuration, modify that set. However, decisions on the addition or deletion of each arc are made locally and regardless of any potential effects on the entire graph. This can lead to a non-optimal selection of arcs. It can also be practically observed that these methods may be unsuccessful in extracting a simple dependency graph, which can result in a non-simple (complex or spaghetti-like) process model. In addition, the solutions employed in these methods to ensure that in the output graph, all tasks are on a path from the initial task to the final task can lose efficiency in the presence of loops.  

On the other hand, in the current heuristic-based discovery methods, there are limitations on utilizing many types of domain knowledge and offering flexibility in the mining procedure, while these options are highly applicable in real-world problems. 

Given the fact that mathematical programming can be very efficient in resolving the mentioned flaws, an approach is proposed in this section to transform the problem of dependency graph discovery into an integer linear programming problem. Thus, in the rest of this section, we first explain the notations and model variables utilized in this paper in Tables~\ref{tab:1} and ~\ref{tab:2},  and then introduce the proposed objective function and the relevant constraints.

The proposed ILP model is totally different from the ILP model introduced by \cite{Ref10}. Our objective function is based on dependency measure, which is superior to the proximity score employed in the objective function of \cite{Ref10} in identifying concurrent tasks, length-two loops, and noises (as is mentioned in Section~\ref{sec2.1}). Also, except for the constraints defining the initial and final tasks, entirely different and more comprehensive constraints have been used in this study. For example, even in the presence of loops, our constraints can ensure that all tasks are on a path from the initial to the final tasks. Furthermore, they can control the arc number and minimum dependency thresholds for the arcs/loops of the output graph.

\begin{table} 
\begin{center}
\caption{The notations used in the rest of the paper (assuming an input event log $L$)}
\label{tab:1} 
\begin{tabular}{p{0.15\linewidth} | p{0.75\linewidth}}

 \hline\noalign{\smallskip}
 Notation & Description  \\
\noalign{\smallskip}\hline\noalign{\smallskip}
$a_{sL} \in A_L$		&	The common initial task of all traces in  $L$ (If there is not such a task, it can be added artificially).\\
$a_{eL} \in A_L$		&	The common final task of all traces in  $L$ (If there is not such a task, it can be added artificially).\\
$d_{i,j}$		&	The measure of dependency between tasks $i \in A_L$ and $j \in A_L$.\\
$s_{i}$		&	The dependency measure for short loop of task $i \in A_L$.\\
$l_{i,j}$		&	The dependency measure for loop of length two between tasks $i \in A_L$ and $j \in A_L$.\\
$M$ 			&	A sufficiently large number.\\
$DepThresh$	&	A user-defined threshold for the minimum dependency measure of arcs that can be present in extracted dependency graphs.\\
$SLoopThresh$	&	A user-defined threshold for the minimum dependency measure of short loops that can be present in the extracted dependency graph.\\
$LoopThresh$	&	A user-defined threshold for the minimum dependency measure of length-two loops that can be present in the extracted dependency graph.\\
$MaxArcsRatio$	&	A user-defined parameter for determining the maximum number of arcs in the dependency graph.\\
$MaxOutputs$	&	A user-defined parameter for determining the maximum number of output arcs for each task in the dependency graph.\\
$MaxInputs$		&	A user-defined parameter for determining the maximum number of input arcs for each task in the dependency graph.\\
\noalign{\smallskip}\hline
\end{tabular}
\end{center}
\end{table}

\begin{table}
\caption{The variables used in the proposed model}
\label{tab:2} 
 \begin{tabular}  {p{0.10\linewidth}  | p{0.80\linewidth}} 
  \hline\noalign{\smallskip}
 Notation & Description  \\
\noalign{\smallskip}\hline\noalign{\smallskip}
$E_{i,j}$	 &	A binary variable that is equal to one if task $i$ is a direct prerequisite for task $j$ (and, hence, there is an arc from $i$ to $j$ in the dependency graph).\\
$x_{i,j}$	&	A binary variable that determines the arcs present in a subgraph of the dependency graph. if $x_{i,j}$ is equal to one, that means that in the subgraph, there is an arc from task $i$ to task $j$. (this variable is used in the constrains pertaining to the trueness of condition 4).\\
$y_{i,j}$	&	A binary variable that determines the arcs present in another subgraph of the dependency graph. if $y_{i,j}$ is equal to one, that means that in the subgraph, there is an arc from task $i$ to task $j$. (this variable has been used in the constrains pertaining to the trueness of condition 5).\\
$R_{i,j}$	&	A binary variable indicates whether $i$ and $j$ are involved in a length-two loop.\\
$u_i$		&	An integer variable that is allocated to each task in order to create a subset of $G_1 =(A_{DG},D_1)$  without any loops ((this variable has been used in the constrains pertaining to the trueness of condition 4.2).\\
$q_i$	&	An integer variable that is allocated to each task in order to create a subset of  $G_2 =(A_{DG},D_2)$  without any loops (this variable has been used in the constrains pertaining to the trueness of condition 5.2).\\
$forced_{i,j}$	&	A binary variable that should be equal to one if both of the following conditions are met: 1) the dependency measure of task $i$ regarding task $j$ ($d_{i,j}$ ) is lower than the user-defined threshold, called $DepThresh$, 2) in order to avoid infeasibility, the proposed mathematical model has set $E_{i,j}$  to one.\\
$forcesl_i$	&	A binary variable that should be equal to one if both of the following conditions are met: 1) the short loop dependency measure of task $i$ ($s_i$ ) is lower than the user-defined threshold, called $SLoopThresh$, 2) to avoid infeasibility, the proposed mathematical model has set $E_{i,i}$  to one.\\
$forcel_{i,j}$	&	A binary variable that should be equal to one if both of the following conditions are met: 1) the length-two loop dependency measure of tasks $i$ and $j$ ($l_{i,j}$ ) is lower than the user-defined threshold, called $LoopThresh$, 2) to avoid infeasibility, the proposed mathematical model has set $R_{i,j}$  to one.\\
\noalign{\smallskip}\hline
  \end{tabular}
    \end{table} 
\smallskip

\subsection{Objective Function}\label{sec3.1}
The proposed objective function aims to maximize the summation of the dependency measures for the arcs/loops present in the extracted dependency graph and uses the following formulation:

\begin{align}
Max \hspace{0.5cm} \sum_{\forall i,j \in A_L:i \ne j}{E_{i,j} \times d_{i,j}}+ \alpha \sum_{i \in A_L}{E_{i,j} \times s_i}
+ \beta/ 2 \sum_{\forall i,j \in A_L: i \ne j}{R_{i,j} \times l_{i,j}}-punishment
\label{eq1}
\end{align}

Where $E_{i,j}$  is the binary decision variable pertaining to the presence of an arc in the dependency graph. If this variable is equal to one, there is an arc from task $i$ to task $j$ in the resultant dependency graph. However, if it is equal to zero, it means that such an arc does not exist in the graph. Obviously, if $E_{i,i}$  is equal to one, $i$ is involved in a short loop. Furthermore, $R_{i,j}$  is the decision variable pertaining to the existence of a length-two loop in the graph. If this variable is equal to one, there is an arc from task $i$ to task $j$ and vice versa from task $j$ to the task $i$. If this variable is equal to zero, these two arcs cannot exist simultaneously in the graph. Moreover, $d_{i,j}$  is the measure of dependency between tasks $j$ and $j$, while $s_i$  and $l_{i,j}$  are dependency measures for short and length-two loops, respectively. They are the scores given to each of the arcs, short and length-two loops in the graph. The formula proposed by \cite{Ref3,Ref6} is employed to determine these measures as follows:
\begin{align}
d_{i,j}=\frac{\mid i>j \mid - \mid j>i \mid}{\mid i>j \mid + \mid j>i \mid+1} \hspace{0.2cm} \forall i,j \in A_L : i \ne j
\label{eq2}
\end{align}
\begin{align}
s_i=\frac{\mid i>i \mid }{\mid i>i \mid +1} \hspace{0.5cm} \forall i \in A_L 
\label{eq3}
\end{align}
\begin{align}
l_{i,j}=\frac{\mid i>>j \mid + \mid j>>i \mid}{\mid i>>j \mid + \mid j>>i \mid+1} \hspace{0.2cm} \forall i,j \in A_L : i \ne j
\label{eq4}
\end{align}
$\alpha$  and $\beta$   are user-defined parameters used to enhance the flexibility of the mining procedure. The lower the $\alpha$  and $\beta$ , the less the probability of the existence of respectively short loops and length-two loops in the extracted dependency graph.
Punishment is a term included in the objective function to penalize non-compliance with the minimum thresholds set by the user for the dependency measures of arcs and short/length-two loops. It is computed according to the following formula:

\begin{align}
punishment=M \sum_{\forall i,j \in A_L:i \ne j}{forced_{i,j} \times (1-d_{i,j})} \nonumber \\
+M \sum_{\forall i,j \in A_L:i \ne j}{forcel_{i,j} \times (1-l_{i,j})}\nonumber \\
+M \sum_{\forall i \in A_L}{forcesl_{i} \times (1-s_{i})}
\label{eq5}
\end{align}

\subsection{Model Constraints}\label{sec3.2}
This section presents the model constraints. For this purpose, first, the constraints are discussed, and some conditions that are intended to be met in the extracted dependency graph are described, then, in Section~\ref{sec3.3}, the mathematical representation of them is introduced. Conditions consist of four different types, which are as follows: 
\begin{itemize}
\item Conditions related to the definition of dependency graphs
\item Conditions related to controlling the simplicity/ complexity of outputs
\item Conditions related to determining minimum thresholds for dependencies
\end{itemize}

In the following, conditions related to each of the aforementioned types are described.

\subsubsection{Conditions for the Definition of Dependency Graphs}\label{sec3.2.1}
Based on the definition presented in Tables~\ref{tab:1} and ~\ref{tab:2}, the following conditions are determined for the dependency graph $DG=(A_{DG},D)$ :

\textit {Condition 1.} The graph must have an initial task.

\textit {Condition 2.} The graph must have a final task.

\textit {Condition 3.} All tasks must be on a path from the initial to the final task. In other words, the following conditions must be met:

 \hangindent=1cm \textit {Condition 3.1.} There must be at least one path from the initial task to all the other tasks.

 \hangindent=1cm \textit {Condition 3.2.} There must be at least one path from all the other tasks to the final task.

Theorems \ref{thm1} and \ref{thm2} prove that condition 3.1 is met if the following condition is true:

\textit {Condition 4. (An alternative for condition 3.1):} There exists at least one directed graph $G_1=(A_{DG},D_1)$   meeting the following conditions:

 \hangindent=1cm \textit {Condition 4.1.} $D_1 \in D$.

 \hangindent=1cm \textit {Condition 4.2.} $G_1$  contains no loop.

 \hangindent=1cm \textit {Condition 4.3.} $\forall a \in A_{DG}-\{ a_s \}, \exists a_1 \in A_{DG} \mid (a_1,a) \in D_1$ .

\begin{theorem}[]\label{thm1}
 If Conditions 4.1 and 4.2 are true, and there is at least one path from $a_1 \in A_{DG} -\{a_s\}$  to  $a \in A_{DG} -\{a_s\}$; then there is a path $T_{a_1,a}$  with a finite length, which is the longest path that includes $a_1$  and ends in $a$.
\end{theorem}
\begin{proof}[Proof of Theorem~{\upshape\ref{thm1}}:]
Obviously, if there is a path from $a_1$  to $a$ , then there is a path  $T_{a_1,a}$  which is the longest path that includes $a_1$  and ends in $a$. Given the fact that $A_{DG}$  is a finite set,  $T_{a_1,a}$ can have an infinite length only if $G_1$ contains loops. This contradicts Condition 4.2.; therefore, the length of $T_{a_1,a}$  must be finite.
\end{proof}
\begin{theorem}[]\label{thm2}
 If Condition 4 is met, Condition 3.1. is also true.
\end{theorem}
\begin{proof}[Proof of Theorem~{\upshape\ref{thm2}}:]
Obviously, if Condition 4.3. is met, for $a \in A_{DG} -\{a_s\}$ , there is at least one task $a_1 \in A_{DG}$ which is included in a path that starts from $a_1$  and ends in $a$. Therefore, according to Theorem \ref{thm1}, there is also a path $T_{a_1,a}$  with a finite length which is the longest path that includes $a_1$  and ends in $a$ . Assuming $i(T_{a_1,a})$  as the start task of  $T_{a_1,a}$, it is obvious that $i(T_{a_1,a})$  must not have any input arcs; because otherwise, $T_{a_1,a}$  cannot be the longest path that includes $a_1$  and ends in  $a$. Thus, $i(T_{a_1,a})$  cannot be any task other than $a_s$  because otherwise, this contradicts Condition 4.3. Therefore, it is proved that if Condition 4 is met, then in $G_1$ , there is at least one path from $a_s$  to every $a \in A_{DG}-\{a_s\}$. Since $D_1$  is a subset of $D$ , it can be concluded that if Condition 4 is met, then in $G$  there is at least one path from $a_s$  to every  $a \in A_{DG}-\{a_s\}$, and Condition 3.1. is met.
\end{proof}

Hence, condition 3.1 can be replaced by condition 4. In a way similar to the proof presented in Theorems \ref{thm1} and \ref{thm2}, it can also be proved that condition 3.2 will be true if the following condition is met.

\textit {Condition 5. (An alternative for condition 3.2):} There exists at least one directed graph $G_2=(A_{DG},D_2)$  meeting the following conditions:

\hangindent=1cm \textit {Condition 5.1.}  $D_2 \in D$.

 \hangindent=1cm \textit {Condition 5.2.}  $G_2$  contains no loop.

 \hangindent=1cm \textit {Condition 5.3.} $\forall a \in A_{DG}-\{ a_e \}, \exists a_1 \in A_{DG} \mid (a,a_1) \in D_2$ .
 
Therefore, if conditions 4 and 5 are met, it is guaranteed that condition 3 is also met. As a result, condition 3 is replaced by conditions 4 and 5. 

\subsubsection{Conditions Associated with the Nature of Length-Two Loops}\label{sec3.2.2}
The following conditions are proposed to control the variable set that indicates the length-two loops. Condition 6 is defined to guarantee the consistency between the variables representing the length-two loops and graph arcs. Condition 7 is derived from flexible configuration options offered by \cite {Ref7}. It reduces the extra behaviors that can be possible according to the extracted dependency graph.

\textit {Condition 6.} The variable showing length-two loop between tasks $i$ and $j$ must be equal to one when they are involved in a length-two loop; otherwise, it is definitely equal to zero. In other words:

 \hangindent=1cm \textit {Condition 6.1}. If there is an arc from $i$ to $j$ and vice versa from $j$ to $i$, the variable showing length-two loop between $i$ and $j$ ($R_{i,j}$ ) must be equal to one.

 \hangindent=1cm \textit {Condition 6.2.} If there is no arc from $i$ to $j$ or from $j$ to $i$, the variable showing length-two loop between $i$ and $j$ ($R_{i,j}$ ) must be equal to zero.

\textit {Condition 7.} The variable showing length-two loop between tasks $i$ and $j$ should be equal to zero if both  $i$ and $j$ are involved in a short loop with themselves.

\subsubsection{Conditions Associated with Controlling the Simplicity/ Complexity of Outputs}\label{sec3.2.3}
These conditions allow users to adjust the simplicity/complexity of output dependency graphs. It is supposed that the presence of many arcs in the model is an indicator of low simplicity and consequently high complexity. Therefore, to control model simplicity/complexity through the idea introduced in \cite {Ref9,Ref15}, the following conditions are proposed for the number of arcs existing in the graph:

\textit {Condition 8.} The total number of arcs in the graph must be smaller or equal to $\lvert A_L \rvert$ multiply by a user-defined threshold named $MaxArcsRatio$.

\textit {Condition 9.} The number of output arcs of each task must be smaller or equal to a user-defined threshold named $MaxOutputs$.

\textit {Condition 10.} The number of input arcs of each task must be smaller or equal to a user-defined threshold named $MaxInputs$.

\subsubsection{Conditions Associated with Determining Minimum Thresholds for Dependencies }\label{sec3.2.4}
These conditions are used to consider the user-defined minimum thresholds for dependency measures during extracting dependency graphs and are defined as follows:

\textit {Condition 11.} As far as possible, each arc in the extracted graph should have a dependency measure higher than a user-defined threshold, called $DepThresh$, unless the desired arc is in a length-two loop.

\textit {Condition 12.} As far as possible, each short loop in the extracted graph should have a short loop dependency measure higher than a user-defined threshold, called $SLoopThresh$.

\textit {Condition 13.} As far as possible, each length-two loop in the extracted graph should have a length-two loop dependency measure higher than a user-defined threshold, called $LoopThresh$.

\subsection{Model Constraints in Mathematical Terms}\label{sec3.3}
The proposed mathematical constraints related to conditions introduced in Section~\ref{sec3.2} are as follows:
\begin{align}
\sum_{\forall i \in A_L}{E_{i,j}}=0 \hspace{0.5cm} \forall j=a_{sL}
\label{eq6}
\end{align}

\begin{align}
\sum_{\forall j \in A_L}{E_{i,j}}=0 \hspace{0.5cm} \forall i=a_{eL}
\label{eq7}
\end{align}

\begin{align}
x_{i,j} \le E_{i,j}  \hspace{0.5cm} \forall i,j \in A_L
\label{eq8}
\end{align}

\begin{align}
u_i-u_j+\lvert A_L \rvert x_{i,j} \le \lvert A_L \rvert-1  \hspace{0.5cm} \forall i,j \in A_L
\label{eq9}
\end{align}

\begin{align}
\sum_{i \in A_L}{x_{i,j}}=1  \hspace{0.5cm} \forall j \in A_L: j \ne a_{sL}
\label{eq10}
\end{align}
\begin{align}
y_{i,j} \le E_{i,j}  \hspace{0.5cm} \forall i,j \in A_L
\label{eq11}
\end{align}

\begin{align}
q_i-q_j+\lvert A_L \rvert y_{i,j} \le\lvert A_L \rvert-1  \hspace{0.5cm} \forall i,j \in A_L
\label{eq12}
\end{align}

\begin{align}
\sum_{j \in A_L}{y_{i,j}}=1  \hspace{0.5cm} \forall i \in A_L: i \ne a_{eL}
\label{eq13}
\end{align}

\begin{align}
R_{i,j} \ge E_{i,j}+E_{j,i}-1   \hspace{0.5cm} \forall i,j \in A_L: i \ne j
\label{eq14}
\end{align}

\begin{align}
2R_{i,j} \le E_{i,j}+E_{j,i}   \hspace{0.5cm} \forall i,j \in A_L: i \ne j
\label{eq15}
\end{align}

\begin{align}
R_{i,j} \le 2- E_{i,i}-E_{j,j}   \hspace{0.5cm} \forall i,j \in A_L: i \ne j
\label{eq16}
\end{align}

\begin{align}
\sum _{i \in A_L}{\sum_{j \in A_L}{E_{i,j}}} \le \lvert A_L \rvert \times MaxArcsRatio
\label{eq17}
\end{align}

\begin{align}
\sum_{j \in A_L}{E_{i,j}} \le MaxOutputs \hspace{0.5cm} \forall i \in A_L
\label{eq18}
\end{align}

\begin{align}
\sum_{i \in A_L}{E_{i,j}} \le MaxInputs \hspace{0.5cm} \forall j \in A_L
\label{eq19}
\end{align}

\begin{align}
E_{i,j} \times DepThresh - R_{i,j} -forced_{i,j} \le max(0, d_{i,j})  \hspace{0.5cm}
 \forall i,j \in A_L: i \ne j
\label{eq20}
\end{align}

\begin{align}
E_{i,i} \times SLoopThresh-forcesl_i \le s_i \hspace{0.3cm} \forall i \in A_L
\label{eq21}
\end{align}

\begin{align}
R_{i,j} \times LoopThresh-forcel_{i,j} \le l_{i,j} \hspace{0.3cm} \forall i,j \in A_L
\label{eq22}
\end{align}

\begin{align}
Binary: E_{i,j},x_{i,j},y_{i,j},R_{i,j},forced_{i,j}\nonumber \\
,forcel_{i,j} \hspace{0.1cm} \forall i,j \in A_L  \nonumber \\
 Binary: forcel_i \hspace{0.5cm} \forall i \in A_L  \nonumber \\
Integer: u_i,q_i \hspace{0.5cm} \forall i,j \in A_L 
\label{eq23}
\end{align}

Constraints (\ref{eq6}–\ref{eq13}) pertain to the definition of dependency graphs. It means that constraint \ref{eq6} determines the initial task (and guarantees that condition (1) is true). This task is the input of the model and should be determined with respect to the event log. If based on the event log, it is impossible to detect a unique initial task, a synthetic initial task is added to the beginning of each trace of events. Similarly, constraint \ref{eq7} determines the final task (and guarantees that condition (2) is met).

Constraints (\ref{eq8}–\ref{eq13}) guarantee that conditions (4) and (5), which are alternatives to condition (3), will be true. In other words, they ensure that all of the tasks will be placed on a path from the initial task to the final task. Condition (4) is incorporated in constraints (\ref{eq8}–\ref{eq10}) in a way that constraint \ref{eq8} determines variables $x_{i,j}$  as a subset of direct prerequisite relations (arcs) of the dependency graph (condition (4.1)). Constraint (9) guarantees that the subset of graph arcs is free of loops (constraint (4.2)). To develop this constraint, an approach available for modeling the traveling salesman problem is utilized. This constraint allocates an integer variable $u_i$  to the task $i$. If there is a path from task $i$ to task $j$ in the graph of this subset, then  $u_i$  value must be smaller than  $u_j$ value. Accordingly, if $i$ and $j$ are placed on one loop of the graph, there is a path both from task $i$ to task $j$ and vice versa. Therefore,  $u_i$  must be smaller than  $u_j$, and at the same time,  $u_j$  must be smaller than  $u_i$. Obviously, this is impossible, meaning this constraint does not allow the graph of the subset to have any loops.
Constraint \ref{eq10} guarantees that all tasks, except for the initial task, will have an input arc in the loop-free subset created by $x_{i,j}$  (condition (4.3)). Similarly, condition (5) was incorporated in constraints (\ref{eq11}–\ref{eq13}).

In addition, constraints (\ref{eq14}–\ref{eq15}) and \ref{eq16} ensure that respectively conditions (6) and (7) are true, whereas constraints (\ref{eq17}–\ref{eq19}) guarantee that conditions (8–10) are met. Finally, constraints (\ref{eq20}-\ref{eq22}) consider meeting conditions (11-13), and constrain \ref{eq23} determines the types of variables used in the model.

\section{Measures to Evaluate Quality of Dependency Graphs}\label{sec4}
So far, various measures have been developed in assessing the quality of process models; however, there are no measures in evaluating the quality of dependency graphs to the best of our knowledge. Thus, in this section, two measures devoted to evaluating dependency graphs are introduced based on two popular measures in assessing the quality of process models, called replay fitness and precision. 

\subsection{Fitness of a Dependency Graph}\label{sec4.1}
The first proposed measure is a modified version of  $PF_{complete}$  that is introduced by \cite {Ref21}. While  $PF_{complete}$  calculates the replay fitness of a process model, the proposed measure calculates the degree of fitness of a dependency graph and the capability of the dependency graph in replaying the log traces.

Assume an event log $L$ and a dependency graph $DG$. After matching each event in $L$ to one task $a \in A_{DG}$ , the replay procedure for each trace $t \in T_L$  and each event  $e_i$ present in $t$ checks if at least an event occurred in $t$ before $e_i$  matches to one of the pre-requisite tasks of the matching task to $e_i$ . If it is true, the occurrence of $e_i$  is compatible with the pre-requisite relations in $DG$. Else, we identify its occurrence as “executed without pre-requisite requirements” and assign a penalty to it. $AEWPr(L,DG)$  represents the number of all events that, according to $DG$ are executed without pre-requisite requirements.

In the same way, for each trace $t \in T_L$  and each event $e_i$  present in $t$, the replay procedure also checks the fulfillment of the post-requisite requirements in $D$G. If the execution of $e_i$   is not compatible with the post-requisite relations of $DG$, we identify its occurrence as “executed without post-requisite requirements” and assign a penalty to it. $AEWPo(L,DG)$  represents the number of all events that according to $DG$ are executed without post-requisite requirements.

If an event is compatible with both pre-requisite and post- requisite relations of $DG$, it is identified as being compatible with $DG$; this means that $DG$ can replay it without any problem. If all events in a trace are compatible with $DG$, then $DG$ can replay the trace without any issue. In this case, the trace is identified as being compatible with $DG$.

Suppose $AFE(L,DG)$   as all events in $L$ compatible with $DG$. Also suppose $NTEWPr(L,DG)$  and $NTEWPo(L,DG)$  as the number of traces in $L$ that contains at least one event non-compatible with respectively pre-requisite and post- requisite relations of $DG$, $NEL(L)$  as the number of all events present in $L$, and $NTL(L)$  as all traces present in $L$. Our proposed measure of the fitness of $L$ regarding $DG$, called $FiM$, can be obtained as follows:

\begin{align}
FiM(L,DG)=\frac{AFE(L,DG)-penalty(L,DG)}{NEL(L)} 
\label{eq24}
\end{align}

Where:

\begin{align}
penalty(L,DG)=\frac{AEWPr(L,DG)}{NTL(L)-NTEWPr(L,DG)+1} 
+\frac{AEWPo(L,DG)}{NTL(L)-NTEWPo(L,DG)+1} 
\label{eq25}
\end{align}

The pseudo-code of the proposed procedure for replaying $L$  on $DG$  and achieving $FiM$  measure is explained in Algorithm \ref{algo1}. Assuming an event log $L$  and a dependency graph $DG$ , the following notations are used in  Algorithm  \ref{algo1}:
\begin{itemize}
\item $\varepsilon$: The set of all events that are present in  $L$.
\item $A_L$: The set of tasks that are present in $L$.
\item $A_{DG}$: The set of tasks that are present in $DG$.
\item $T_L$: The set of traces that are present in $L$.
\item $E:t \in T_L \mapsto E_t \subseteq \varepsilon$: The set of events that are present in  $t$.
\item $\delta :a \in A_L \mapsto A_{DG}$: The member of $A_{DG}$  that corresponds to $a \in A_L$.
\item $Act :e_i \in \varepsilon \mapsto a \in A_L$: The member of $A_L$  that corresponds to event $e_i \in \varepsilon$.
\item $inp : a \in A_{DG} \mapsto i_a  \subseteq A_{DG}$: The set of tasks that according to $DG$  are pre-requisite of  $a \in A_{DG}$.
\item $out : a \in A_{DG} \mapsto o_a  \subseteq A_{DG}$: The set of tasks that according to $DG$  are post-requisite of  $a \in A_{DG}$.
\item $pre: e_i \in E(t) \mapsto a_{pr}  \subseteq A_L$: The sequence of tasks which are mapped to the sequence of all events present in $t$  that occurred before $e_i$.
\item $suc: e_i \in E(t) \mapsto a_{po}  \subseteq A_L$: The sequence of tasks which are mapped to the sequence of all events present in $t$  that occurred after $e_i$.
\end{itemize}

\begin{algorithm}
\caption{Pseudo-code of replaying an event log $L$  on a dependency graph $DG$, and obtaining values required for calculation of $FiM$ measure}\label{algo1}
\KwIn {Event log $L$}
\KwOut {$FiM$}

 $AEWPr \gets 0$\;
$AEWPo \gets 0$\;
 $AFE \gets 0$\;
 $NTEWPr \gets 0$\;
 $NTEWPo \gets 0$\;
 $NEL \gets 0$\;
 $NTL \gets 0$\;
\For{\texttt{each $t \in T_L$}} {
	 $NTL \gets NTL+1$\;
	 $EWPr \gets 0$\;
	 $EWP0 \gets 0$\;
	\For{\texttt{each event $e_i \in E(t)$}}{
		 $NEL \gets NEL+1$\;
		\If{$(\nexists e \in E(t) \mid Act(e) \in pre(e_i) \wedge \delta (Act(e)) \in inp(\delta(Act(e_i))) ) \wedge not(inp(\delta(Act(e_i)))=\emptyset)$}{
           		  $AEWPr \gets  AEWPr + 1$    \;    
			 $EWPr \gets 1$   \; 
        		}
		\If{$(\nexists e \in E(t) \mid Act(e) \in suc(e_i) \wedge \delta (Act(e)) \in out(\delta(Act(e_i))) ) \wedge not(out(\delta(Act(e_i)))=\emptyset)$}{
           		  $AEWPo \gets AEWPo + 1$     \;   
			 $EWPo \gets 1$    \;
        		}
		\If{$EWPr+EWPo=0$}{
           		  $AFE \gets AFE + 1$        \;
        		}
	}
	\If{$EWPr=1$}{
           	 $NTEWPr \gets NTEWPr + 1$     \;   
        	}
	\If{$EWPo=1$}{
           	 $NTEWPo \gets NTEWPo + 1$      \;  
        	}
}
 $penalty \gets \frac{AEWPr}{NTL-NTEWPr+1}+\frac{AEWPo}{NTL-NTEWPo+1}$\;
 $FiM \gets \frac{AFE-penalty}{NEL}$\;
\Return{$FiM$}\;
\end{algorithm}

\subsection{Precision of a Dependency Graph}\label{sec4.2}

Our second proposed measure is named $PrM$. It is a modified version of the “Advanced behavioral appropriateness” measure introduced by \cite {Ref22}, which evaluates the precision of a process model. The higher the “Advanced behavioral appropriateness” of the model, the less the allowance to the behaviors not present in the event log. In the same way, $PrM$ takes higher values when the dependency graph contains fewer dependencies/long-distance dependencies not present in the event log. 

Assume an event log $L$ and a dependency graph $DG$. To calculate $PrM$ we analyze and compare the behaviors that are potentially possible in $DG$ with the behaviors that are actually observed in $L$. Suppose $F^L(x,y)$  and $F^{DG}(x,y)$  as follows:

\begin{itemize}
\item $F^L(x,y)$: Assuming tasks $x,y \in A_L$, $F^L(x,y)$  equals one if there exists at least one trace $t \in T_L$ , in which event $e_1$  is executed, and then event $e_2$  is eventually executed; whereas, $e_1$  and $e_2$  are mapped onto $x$ and $y$ respectively. Otherwise, $F^L(x,y)$  equals zero.
\item  $F^{DG}(x,y)$: Assuming tasks $x,y \in A_L$, $F^{DG}(x,y)$  equals one if there exists at least one path from task $x$ to task $y$ in $DG$. Otherwise,  $F^{DG}(x,y)$   equals zero.
\end{itemize}

Let $A=A \cup A_{DG}$ , then $PrM$ can be achieved as follows:

\begin{align}
PrM(L,DG)=\frac{\sum_{x,y \in A}{(F^L(x,y) \times F^{DG}(x,y))}}{ F^{DG}(x,y)}
\label{eq26}
\end{align}

Where $F^{DG}(x,y)$ can be calculated by Warshall’s algorithm \cite {Ref23}.  The pseudo-code for calculating $F^L(x,y)$ is presented in Algorithm \ref{algo2}.

\begin{algorithm}
\caption{Pseudo-code for calculation of $F^L(x,y)$  (required for achieving $PrM$ measure)}\label{algo2}
\KwIn {Event log $L$}
\KwOut  {$F^L(x,y)$ } 

 $F^L(x,y) \gets 0   \hspace{0.5cm}  // \forall x,y \in A_L$\;
\For{\texttt{each $t \in T_L$}}{
	\For{\texttt{each event $e_i \in E(t)$}}{
		 $F^L(Act(e_i),a) \gets 1  \hspace{0.5cm} // \forall a \in a_L \mid a \in suc(e_i)$\;
	}
}

\Return{$F^L(x,y)$}\;
\end{algorithm}

\section{Experimental Results}\label{sec5}

This section evaluates the proposed ILP model and presents numerical results. The results of the proposed method were compared to the results of the dependency graph discovery step of the most prominent heuristic mining methods, i.e., Heuristics Miner (HM)\cite {Ref3}, Flexible Heuristics Miner (FHM)\cite {Ref6}, Fodina\cite {Ref7}. The dependency graph discovery step of Proximity Miner (PM) \cite {Ref10} was also included in our experiments because the outputs of this method can be used for applying the second step of heuristic mining methods. To apply HM and FHM, “Mine for a Heuristics Net Using Heuristics Miner” and “Mine for a Causal Net Using Heuristics Miner” plugins for ProM 6.4 were used, respectively. “Mine Causal Net with Fodina” plugin for ProM6.6 was also used to apply Fodina. To apply PM, the ILP model for dependency graph discovery proposed by the introduced study was coded in GAMS software. 

The dependency graphs extracted by HM, FHM, Fodina, and the proposed method depend on the methods’ parameters. However, since each method has many parameters, comparing all methods with different combinations of all of their parameters was extremely time-consuming and exhausting; therefore, we decided to make comparisons by using different values of only the most important parameter of each method as it was observed that the other parameters have a much smaller impact. Thus, in our evaluations, we considered 11 different values for the most important parameter of each method, as is presented in Table ~\ref{tab:3}. Parameters disregarded in the table were determined by the software default. Note that most existing studies that have used heuristic-based process discovery methods in their evaluations utilized only one set of parameters (one output) for each method (for example: \cite {Ref7,Ref19,Ref24, Ref25}).

In the case of PM, because no domain knowledge is employed in the evaluations, there are no parameters in the method that can be changed. Hence for each event log, Proximity Miner can achieve only one dependency graph, and as a result, it is excluded in Table~\ref{tab:3}.

\begin{table}
\begin{center}
\caption{Configurations applied to each method (Proximity Miner is excluded since it has no adjustable parameter)}
\label{tab:3}       
\begin{adjustbox}{width=0.9\textwidth}
\begin{tabular}{lccccccccccc}
\hline\noalign{\smallskip}
	&	\multicolumn{11}{c}{Configurations}\\
	\noalign{\smallskip}\cline{2-12}\noalign{\smallskip}
 & C1	&	C2	&	C3	&	C4	&	C5	&	C6	&	C7	&	C8	&	C9	&	C10	&	C11  \\
\noalign{\smallskip}\cline{2-12}\noalign{\smallskip}
\underline {HM}\\
Dependency	&		80	&	82	&	84	&	86	&	88	&	90	&	92	&	94	&	96	&	98	&	100\\
\\
\underline {FHM}\\
Dependency	&		80	&	82	&	84	&	86	&	88	&	90	&	92	&	94	&	96	&	98	&	100\\
\\
\underline {Fodina}\\
Dependency threshold	&		80	&	82	&	84	&	86	&	88	&	90	&	92	&	94	&	96	&	98	&	100\\
\\
\underline {Proposed Method}\\												
MaxArcsRatio		&	2.1	&	2	&	1.9	&	1.8	&	1.7	&	1.6	&	1.5	&	1.4	&	1.3	&	1.2	&	1.1\\
DepThresh	&		0	&	0	&	0	&	0	&	0	&	0	&	0	&	0	&	0	&	0	&	0\\
SLoopThresh	&		0	&	0	&	0	&	0	&	0	&	0	&	0	&	0	&	0	&	0	&	0\\
LoopThresh	&		0	&	0	&	0	&	0	&	0	&	0	&	0	&	0	&	0	&	0	&	0\\
MaxOutputs		&		1000	&	1000	&	1000	&	1000	&	1000	&	1000	&	1000	&	1000	&	1000	&	1000	&	1000\\
MaxInputs		&		1000	&	1000	&	1000	&	1000	&	1000	&	1000	&	1000	&	1000	&	1000	&	1000	&	1000\\

\noalign{\smallskip}\hline
\end{tabular}
\end{adjustbox}
\end{center}
\end{table}

To compare the quality of the obtained dependency graphs, we examined the fitness, precision, and simplicity of the graphs in this section. Fitness and precision have been discussed in Section~\ref{sec4}. Simplicity is also one of the main dimensions of assessing process model quality \cite {Ref8}. According to the simplicity dimension, process models should be as simple as possible, and the simplest process model that can properly describe the event log behaviors is desired. According to  \cite {Ref24}, the size of process models is a measure of simplicity. Increasing the size of dependency graphs usually makes the dependency graphs visually less simple and raises the size of process models. Thus, considering the size of dependency graphs as a measure of dependency graph simplicity, the simpler dependency graphs are desired; because they usually result in simpler process models, and a non-simple dependency graph cannot lead to a simple process model. 

The size of dependency graphs is described through the number of graph nodes and arcs. For an event log $L$, the node set of the dependency graphs extracted by all methods used in the evaluations is equal to $A_L$. Hence, for an event log $L$, the node numbers of results of all methods are equal. In consequence, only the arc number of dependency graphs, called $A_N$, is considered as the measure of dependency graph simplicity in the evaluations of this paper. the higher the $AN$, the lower the dependency graph simplicity.

 The $FiM$ and $PrM$ measures introduced in the previous section were used to evaluate the fitness and precision of the dependency graphs. However, these two measures are complementary, and a graph with high fitness but low precision, or vice versa, is not considered as a high-quality graph. To consider the fitness and precision measures simultaneously, according to \cite {Ref8}, we used F-score as the harmonic mean of fitness ($FiM$) and precision ($PrM$). The $F-score$ measure takes a low value when one of $FiM$ and $PrM$ is low, even when the other measure is high.

For our evaluations, we used all publicly available event logs used by the review and benchmark paper \cite {Ref26} (consists of 12 event logs), as well as, Production \cite {Ref27} and Receipt phase of an environmental permit application process (Receipt) \cite {Ref28} event logs. These event logs cover the various domains of healthcare, finance, production, government services, and IT management. 10 out of 14 event logs consist of BPI Challenge (BPIC) logs, while all event logs are real and publicly available at 4TU Centre for Research Data\footnote{https://data.4tu.nl/}. More details on the specifications of the event logs are given in Table~\ref{tab:4}. It should be mentioned that wherever an event log had not a unique initial/final task, the artificial initial/final tasks were added to it by the “Add Artificial Events” plugin for ProM 6.6.

\begin{table}
\begin{center}

\caption{Specifications of the event logs used in the assessments}
\label{tab:4}       
\begin{adjustbox}{width=0.5\textwidth}
\begin{tabular}{llll}
\hline\noalign{\smallskip}
Log Name & Number of &  $Number of $ &  $Number of $  \\
 	&  Activities &  $ Traces$ &  $Events$  \\
\noalign{\smallskip}\hline\noalign{\smallskip}
$BPIC12$ &	36&	13087& 	262200\\
$BPIC13_{cp}$		&	7&	1487&	6660\\
$BPIC13_{inc}$	&	13&	7554& 	65533\\
$BPIC14_f$	&	9&	41353& 	369485\\
$BPIC15_{2f}$	&	82&	681&	24678\\
$BPIC15_{3f}$	&	62&	1369&	43786\\
$BPIC15_{4f}$	&	65&	860&	29403\\
$BPIC15_{5f}$	&	74&	975&	30030\\
$BPIC17_f$	&	41&	21861& 	714198\\
$RTFMP$	&	11&	150370&	561470\\
$Sepsis$	&	16&	1050&	 15214\\
$receipt$	&	27&	1434& 	8577\\
$Production$	&	55&	225&	4543\\
\noalign{\smallskip}\hline
\end{tabular}
\end{adjustbox}
\end{center}
\end{table}

\subsection{Assessments}\label{sec5.1}

To make our comparisons, for each event log and each method except for PM, the configurations C1 to C11 were utilized according to Table~\ref{tab:3}. As a result, for each event log, 11 dependency graphs were extracted by each method. Since for PM, no adjustable parameter was available; only one dependency graph was extracted by PM for each event log. Hence, we extracted 45 dependency graphs for each event log (For $BPIC15_{3f}$ and $BIC17_f$ event logs we extracted 34 dependency graphs because the plugin used for applying FHM failed to generate a model from these event logs). In consequence, we extracted and evaluated a total of 608dependency graphs for all event logs in our assessments. It should be mentioned that the dependency graphs in which at least one task was not on a path from the initial task to the final task were excluded from the experiments.

In the first experiment, the complexity rates of dependency graphs developed by each method were analyzed to compare the lowest $AN$ (complexity) that each method can achieve. Since achieving a dependency graph with low $AN$ and also a low $F-score$ is worthless, we defined four thresholds for minimum acceptable $F-score$ measures (i.e., 0.6, 0.7, 0.8, 0.9) that dependency graphs can attain. Then for each threshold, after excluding the dependency graphs with $F-score$ measures lower than the threshold, we compared the lowest $AN$ that each method could achieve. The results can be observed in Table~\ref{tab:5}. We used “–“ wherever each method did not achieve a dependency graph with an $F-score$ higher than the specified threshold.

\begin{table}
\caption{The minimum $AN$ that each method achieved for different $F-score$ minimum acceptable thresholds}
\label{tab:5}       
\begin{adjustbox}{width=1\textwidth}
\begin{tabular}{lcccccccccccccc}
\hline\noalign{\smallskip}
	&	\multicolumn{11}{c}{Event Logs}\\
\noalign{\smallskip}\cline{2-15}\noalign{\smallskip}
	&	\rotatebox{90}{$BPIC12$} &	\rotatebox{90}{$BPIC13_{cp}$} &	\rotatebox{90}{$BPIC13_{inc}$}  &	\rotatebox{90}{$BPIC14_f$} &	\rotatebox{90}{$BPIC15_{1f}$} &	\rotatebox{90}{$BPIC15_{2f}$} &	\rotatebox{90}{$BPIC15_{3f}$} &	\rotatebox{90}{$BPIC15_{4f}$} &	\rotatebox{90}{$BPIC15_{5f}$} &	\rotatebox{90}{$BPIC17_f$} &	\rotatebox{90}{$RTFMP$} &	\rotatebox{90}{$Sepsis$} &	\rotatebox{90}{$Receipt$}&  \rotatebox{90}{$Production$}	  \\		
\noalign{\smallskip}\cline{2-15}\noalign{\smallskip}
\underline {$F-score$ minimum threshold=0.60}\\
HM	&		58 &	16 &	33 &	18 &	87 &	100 &	76 &	79 &	88 &	63 &	16 &	31 &	42 &	-\\
FHM	&		112 &	- &	- &	22 &	123 &	183 &	- &	147 &	152 &	- &	25 &	40 &	47 &	-\\
Fodina	&	60 &	17 &	44 &	23 &	94 &	110 &	88 &	84 &	95 &	88 &	19 &	36 &	44 &	140\\
PM	&		201 &	32 &	98 &	37 &	124 &	197 &	189 &	151 &	154 &	117 &	77 &	135 &	113 &	-\\
The Proposed Method	&		{\bf 51}	&	{\bf 8}	&	{\bf 21}	&	{\bf 11}	&	{\bf 84}	&	{\bf 98}	&	{\bf 74}	&	{\bf 78}	&	{\bf 81}	&	{\bf 46}	&	{\bf 13}	&	{\bf 27}	&	{\bf 33} &	{\bf 102}\\
\\
\underline {$F-score$ minimum threshold=0.70}\\													
HM	&		58 &	16 &	33 &	18 &	87 &	147 &	{\bf 76} &	138 &	88 &	63 &	16 &	31 &	42 &	-\\
FHM	&		112 &	- &	- &	22 &	123 &	183 &	- &	147 &	152 &	- &	25 &	40 &	47 &	-\\
Fodina	&	60 &	17 &	44 &	23 &	94 &	110 &	88 &	138 &	95 &	88 &	19 &	36 &	44 &	-\\
PM	&		201 &	32 &	98 &	37 &	124 &	197 &	-&	151 &	154 &	117 &	77 &	135 &	- &	-\\
The Proposed Method	&		{\bf 51}	&	{\bf 8}	&	{\bf 21}	&	{\bf 11}	&	{\bf 84}	&	{\bf 98}	&	80	&	{\bf 84}	&	{\bf 88}	&	{\bf 50}	&	{\bf 13}	&	{\bf 27}	&	{\bf 33} & -\\
\\														
\underline {$F-score$ minimum threshold=0.80}\\														
HM	&		58 &	16 &	33 &	18 &	87 &	147 &	141 &	146 &	88 &	63 &	16 &	31 &	44 & -\\
FHM	&		112 &	- &	- &	22 &	123 &	183 &	- &	147 &	- &	- &	25 &	40 &	47 &	-\\
Fodina	&	60 &	17 &	44 &	23 &	94 &	110 &	88 &	146 &	95 &	88 &	19 &	36 &	44 &	-\\
PM	&		201 &	32 &	98 &	37 &	124 &	197 &	-&	151 &	-&	117 &	77 &	135 &	- &	-\\
The Proposed Method	&		{\bf 51}	&	{\bf 8}	&	{\bf 22}	&	{\bf 12}	&	{\bf 84}	&	{\bf 98}	&	{\bf 80}	&	{\bf 91}	&	{\bf 88}	&	{\bf 50}	&	{\bf 13}	&	{\bf 28}	&	{\bf 39} & -\\
\\														
\underline {$F-score$ minimum threshold=0.90}\\														
HM	&		{\bf 58} &	16 &	- &	18 &	- &	-&	- &	- &	- &	63 &	16 &	31 &	- & -\\
FHM	&		- &	- &	- &	22 &	- &	- &	- &	- &	- &	- &	25 &	40 &	- &	-\\
Fodina	&	60 &	17 &	{\bf 44} &	23 &	- &	- &	- &	- &	- &	88 &	26 &	36 &	- &	-\\
PM	&		- &	32 &	98 &	37 &	- &	- &	-&	- &	-&	117 &- &	- &	- &	-\\
The Proposed Method	&		59	&	{\bf 12}	&	-	&	{\bf 13}	&	-	&	-	&	-	&	-	&	-	&	{\bf 50}	&	{\bf 14}	&	{\bf 30}	&	-  & -\\

\noalign{\smallskip}\hline
\end{tabular}
\end{adjustbox}
\end{table}

According to Table~\ref{tab:5}, in the cases of the minimum acceptable $F-score$ thresholds of 0.6 and 0.8, the simplest dependency graph for all event logs was achieved by the proposed method. In the case of the threshold of 0.7, only for one event log (BPIC153f), the simplest dependency graph was not extracted by the proposed method; however, for this event log, the proposed method could obtain an $AN$ level only 5\% higher than the simplest result. In the case of the minimum acceptable $F-score$ threshold of 0.9, for seven event logs, at least a method could extract a dependency graph with an $F-score$ higher than the acceptable threshold. While, for five event logs out of the 7 logs, the proposed method was successful in extracting the simplest dependency graph. Whereas, for the BPIC12 event log, the proposed method could obtain a dependency graph with only one arc more than the simplest graph. Thus, the proposed method was significantly superior to the other methods in this regard.

Our second experiment included analyzing the cases in which each method extracted a dependency graph that contained at least one task that was not on a path from the initial task to the final task. The results are shown in Table~\ref{tab:6}. It can be seen that this is not a rare problem in FHM and Fodina methods. In addition, despite HM and PM methods did not encounter this issue, theoretically, they are also capable of encountering it, and only the proposed method can mathematically guarantee the non-occurrence of the problem. 

\begin{table}
\begin{center}
\caption{Percentage of the cases each method extracted a dependency graph in which at least one activity was not on a path from the initial task to the final task }
\label{tab:6}       
\begin{tabular}{p{0.30\linewidth}  p{0.05\linewidth}	p{0.05\linewidth}	p{0.08\linewidth}	p{0.05\linewidth}	p{0.08\linewidth}}
\hline\noalign{\smallskip}
	& HM	&	FHM	&	Fodina	& PM&	Proposed Method\\
\noalign{\smallskip}\hline\noalign{\smallskip}
Percentage of the case that each method extracted a dependency graph in which at least an activity that the final activity is not reachable from it, or it is not reachable from the initial activity		&		0\%	&	25\%	&	9\%	&	0\%  & 0\% \\

\noalign{\smallskip}\hline
\end{tabular}
\end{center}
\end{table}

following the superiority of the method in previous experiments, the $FiM$ and $PrM$ that the dependency graphs extracted by each method could obtain were compared in the next experiment. Here again, in order to consider only one measure in assessing the graph qualities, the $F-score$ measure was utilized. For each event log, after assessing the dependency graphs extracted by each method, we considered the dependency graph with the highest $F-score$ as the best result of the method. Then we compared the $AN$, $FiM$, $PrM$, and $F-score$ of the best results obtained by each method (called $AN_b$, $FiM_b$, $PrM_b$, and $F-score_b$, respectively) in Table~\ref{tab:7}. For each event log, we chose the result with the highest $F-score$ as the highest quality result. Wherever there was more than one result with the highest $F-score$, we chose the candidate result with the lowest $AN$ as the highest quality result. For each event log, the highest quality result is boldfaced in Table~\ref{tab:7}. In addition, we used “–“ wherever a method did not extract any dependency graph with all tasks on a path from the initial task to the final task.

\begin{table}
\caption{Quality assessment results of the best dependency graphs extracted by different methods}
\label{tab:7} 

\begin{adjustbox}{width=1\textwidth}
\begin{tabular*}{1.5\textheight}{@{\extracolsep{\fill}}lcccccccccccccccc@{\extracolsep{\fill}}}
\hline\noalign{\smallskip}

	&	&	\multicolumn{2}{c}{HM}&	&	\multicolumn{2}{c}{FHM}&	&	\multicolumn{2}{c}{Fodina}&	&	\multicolumn{2}{c}{PM}&	 &	\multicolumn{2}{c}{Proposed Method}\\
\noalign{\smallskip}\cline{3-4}\cline{6-7}\cline{9-10}\cline{12-13}\cline{15-16}\noalign{\smallskip}
Log Name	&	&	$AN_b$	& $F-score_b$		&	&	$AN_b$	& $F-score_b$	&	&	$AN_b$	& $F-score_b $	&	&	$AN_b$	& $F-score_b)$ &	&	$AN_b$	& $F-score_b $\\
Log Name	&	&		& $ (FiM_b, PrM_b)$		&	&		& $ (FiM_b, PrM_b)$	&	&		& $ (FiM_b, PrM_b)$	&	&		& $ (FiM_b, PrM_b)$  &	&		& $ (FiM_b, PrM_b)$\\
\noalign{\smallskip}\hline\noalign{\smallskip}

$BPIC12$ & &	58 &0.93(0.90, 0.97)& &	112&		0.84(0.98, 0.73)& &	60	&	0.92(0.93, 0.90)& &	201	&	0.83(1.00, 0.72)& &	{\bf 66}	&	{\bf 0.94(0.94, 0.93)}\\
$BPIC13_{cp}$	& &	{\bf 17} &	{\bf 0.95(0.99, 0.92)}	& &	-&	-&&		18&	0.94(0.99, 0.90)& &		32&	0.93(0.99, 0.88)& &		14&	0.92(0.92, 0.91)
\\
$BPIC13_{inc}$		& &	33&	0.89(0.95, 0.84)	&&	-&	-	&&	{\bf 44}&	{\bf 0.92(1.00, 0.85)}&&		98&	0.92(1.00, 0.85)&&		24&	0.84(0.80, 0.88)
\\
$BPIC14_f$		& &	22&	0.99(0.99, 1.00)	&&	22&	0.99(0.99, 1.00)&&		{\bf 33}&	{\bf 1.00(1.00, 1.00)}	&&	37&	1.00(1.00, 1.00)	&&	15&	0.99(0.98, 1.00)
 \\
$BPIC15_{1f}$		& &	113&	0.89 (0.98, 0.82)&&		123&	0.88(1.00, 0.79)&&		113&	0.89(0.98, 0.82)&&		124&	0.88(1.00, 0.79)&&		{\bf 98}&	{\bf 0.89 (0.95, 0.84)} \\
$BPIC15_{2f}$			& &	147&	0.86(0.95, 0.80)	&&	183&	0.80(0.99,0.67)	&&	147&	0.86(0.95, 0.80)&&		197&	0.80(1.00, 0.67)	&&	{\bf 139}&	{\bf 0.89(0.93, 0.86)} \\
$BPIC15_{3f}$			& &	141&	0.81(0.97, 0.70)	&&	-&	-&&		88&	0.85(0.82, 0.88)	&&	189&	0.69(1.00, 0.53)&&		{\bf 105}	& {\bf 0.89(0.90, 0.88)} \\
$BPIC15_{4f}$		& &	138&	0.80(0.99, 0.67)	&&	147&	0.80(1.00, 0.67)	&&	138&	0.80(0.99, 0.67)&&		151&	0.80(1.00, 0.67)	&&	 {\bf 97}&	 {\bf 0.86(0.90, 0.82)}
 \\
$BPIC15_{5f}$		& &	88&	0.84(0.80, 0.88)&&		152&	0.79(1.00, 0.65)&&		95&	0.85(0.84, 0.87)	&&	154&	0.79(1.00, 0.65)	&&	 {\bf 118}&	 {\bf 0.89(0.93, 0.85)}
\\
$BPIC17_f$		& &	86&	0.96(0.97, 0.95)	&&	-&	-	&&	{\bf  88}&	{\bf  0.97(0.99, 0.95)}	&&	117&	0.96(1.00, 0.92)&&		67	&0.96(0.96,0.95)
\\
$RTFMP$		& &	{\bf 24}&	{\bf 0.99(0.97, 1.00)}	&&	25&	0.99(0.97, 1.00)	&&	26&	0.92(0.99, 0.85)	&&	77&	0.88(1.00,0.79)	&&	{\bf 24}	&{\bf 0.99(0.97, 1.00)}
\\
$Sepsis$		& &	35&	0.99(0.98, 1.00)	&&	{\bf  42}&	{\bf  1.00(0.99, 1.00)}	&&	49&	1.00(0.99, 1.00)&&		135&	0.81(1.00, 0.68)	&&	37&	0.97(0.94, 1.00)
\\
$receipt$			& &	44&	0.81(0.99, 0.69)	&&	{\bf  50}&	{\bf  0.82(1.00, 0.69)}	&&	44&	0.81(0.99, 0.69)	&&	113&	0.67(1.00, 0.51)	&&	39&	0.81(0.98, 0.70)
 \\
$Production$			& &	134&	0.54(0.84, 0.39)&&		-	&	-&&	140&	0.60(0.91, 0.44)	&&	433&	0.48(1.00, 0.31)	&&	{\bf 102}&	{\bf 0.62(0.72, 0.54)}
\\
\noalign{\smallskip}\hline

\end{tabular*}
\end{adjustbox}

\end{table}

According to Table~\ref{tab:7}, the proposed method obtained the highest quality result for 8 out of 14 event logs. While the highest quality result was obtained by HM, FHM, Fodina, and PM methods, respectively for 2, 2, 3, and 0 event logs. On the other hand, for $BIC14_f$, $BIC17_f$, and $Receipt$ event logs, the $F-score$ of the best result achieved by the proposed method is only 0.01 lower than the $F-score$ of the highest quality result, whereas it enjoys far less $AN$ (graph complexity). Especially for $BIC14_f$, the $AN$ of the best result achieved by the proposed method is about 55\% lower than $AN$ of the highest quality result. In addition, for $BPIC13_{cp}$ and $Sepsis$ event logs, the $F-score$ of the best result obtained by the proposed method was only 0.03 lower than the $F-score$ of the highest quality result, and here again, the best result of the proposed method has less $AN$ than the highest quality result. Hence, in 13 out of 14 event logs, the best result of the proposed method was the highest quality result; or its $F-score$ was very close to the $F-score$ of the highest quality result while it was significantly simpler than the highest quality result. Only for the $BPIC13_{inc}$ event log, the best result of the proposed method was not the highest quality result nor comparable to the highest quality result. Therefore, the proposed method outperformed the other methods in this experiment too, and in most cases, the best result (in terms of all three measures of fitness, precision, and simplicity) achieved for each event log was obtained by the proposed method. For instance, for $BPIC15_f$ event log, the best result of each method can be seen in Fig. ~\ref {fig:1}.

\begin{figure*}
\begin{center}
  \includegraphics [width=1\textwidth]{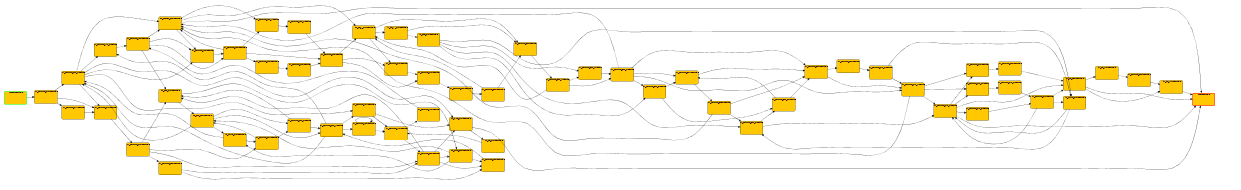}
  \\
  \text{(a)}\\
  \includegraphics [width=1\textwidth]{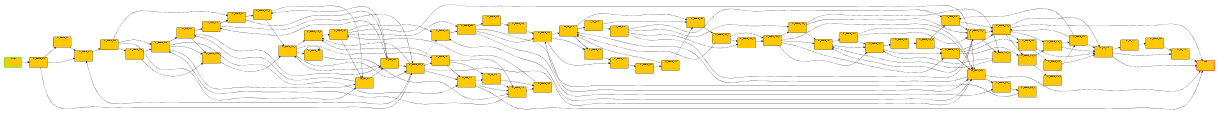}
    \\
  \text{(b)}\\
  \includegraphics [width=1\textwidth]{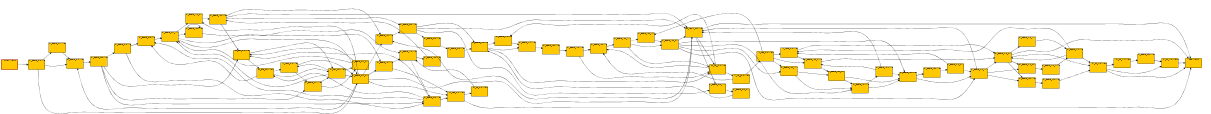}
      \\
  \text{(c)}\\
    \includegraphics [width=1\textwidth]{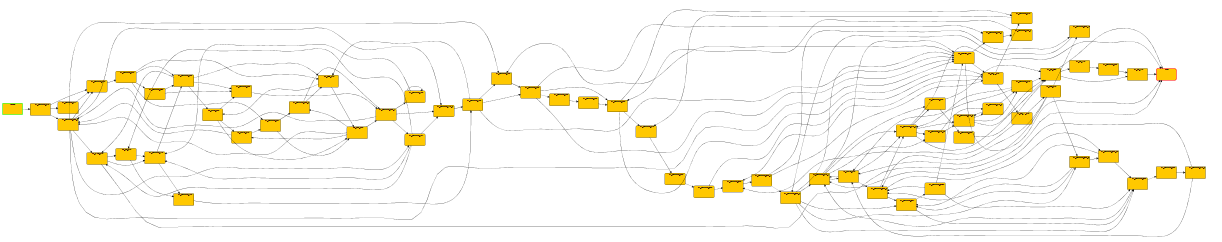}
      \\
  \text{(d)}\\
   \includegraphics [width=1\textwidth]{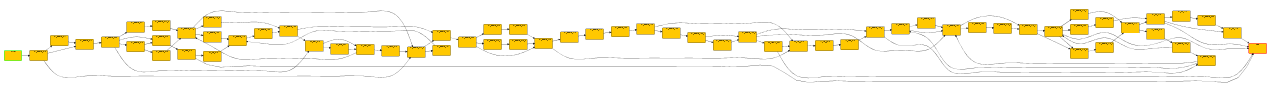}
      \\
\text{(e)}\\
\caption{Best output achieved by (a) HM, (b)FHM, (c) Fodina, (d) PM, and (e) Proposed method for $BPIC15_{4f}$ event log}
\label{fig:1}       
\end{center}
\end{figure*}

In general, the proposed method indicated superiority in the performed experiments. However, using ILP may raise a concern about the solution time of the problem. For event log $L$, the proposed ILP model is composed of  $6\lvert A_L \rvert^2+3\lvert A_L \rvert$  variables and $9\lvert A_L \rvert^2+\lvert A_L \rvert+1$  constraints. Thus, when $\lvert A_L \rvert$  is a large number, in the worst case, the problem-solving will take a very long time. However, usually, process mining problems do not have a high $\lvert A_L \rvert$. According to \cite {Ref14}, most existing process mining researches consider problems with $\lvert A_L \rvert$  levels less than 30-40. The study \cite {Ref14} that is dedicated to discovering the process model of large and complex event logs used event logs with $\lvert A_L \rvert$  levels between 30-130. On the other hand, usually, ILP models are anticipated to be solved in a time very shorter than the worst case. Therefore, the proposed model is expected to be solved at a short or acceptable time for a wide range of levels that process mining problems usually have. To assess this anticipation, experiments on real-world and synthetic event logs were performed. All experiments were done on a PC with Intel(R) Core(TM) i5 CPU @2.40GHz and 8GB RAM.

First, the time spent by the proposed method on discovering the dependency graph of each real event log introduced in Table~\ref{tab:4} was analyzed. The information about the time spent on solving the ILP model for each event log is presented in Table~\ref{tab:8}. According to the results for all cases, the proposed ILP model was successful in solving the model in less than 74 seconds which is considered a relatively high and acceptable speed. Furthermore, it should be noted that for all event logs, the median of time spent on problem-solving was much lower than 74 seconds. For 5 event logs, the median of the spent time was less than 1 second, and for 9 event logs, it was less than 4 seconds.

\begin{table*}
\caption{Time spent on problem-solving by the proposed method (seconds)}
\label{tab:8}       
\begin{adjustbox}{width=1\textwidth}
\begin{tabular}{lcccccccccccccc}
\hline\noalign{\smallskip}
	&	\multicolumn{11}{c}{Event Logs}\\
\noalign{\smallskip}\cline{2-15}\noalign{\smallskip}
	&	\rotatebox{90}{$BPIC12$} &	\rotatebox{90}{$BPIC13_{cp}$} &	\rotatebox{90}{$BPIC13_{inc}$}  &	\rotatebox{90}{$BPIC14_4$} &	\rotatebox{90}{$BPIC15_{1f}$} &	\rotatebox{90}{$BPIC15_{2f}$} &	\rotatebox{90}{$BPIC15_{3f}$} &	\rotatebox{90}{$BPIC15_{4f}$} &	\rotatebox{90}{$BPIC15_{5f}$} &	\rotatebox{90}{$BPIC17_f$} &	\rotatebox{90}{$RTFMP$} &	\rotatebox{90}{$Sepsis$} &	\rotatebox{90}{$Receipt$} &	\rotatebox{90}{$Production$}	  \\		
\noalign{\smallskip}\cline{2-15}\noalign{\smallskip}
Median	&	2.36	&	0.09	&	0.15	&	0.09	&	12.49	&	11.76	&	3.37	&	6.47	&	14.54	&	1.38	&	0.13	&	0.19	&	1.26& 10.24\\
Standard Deviation	&	1.87	&	0.03	&	0.03	&	0.04	&	10.88	&	6.59	&	2.21	&	4.68	&	19.57	&	0.31	&	0.04	&	0.02	&	0.98& 6.97\\
Maximum	&	5.06	&	0.16	&	0.19	&	0.17	&	33.28	&	23.22	&	9.59	&	14.94	&	73.58	&	2.17	&	0.19	&	0.23	&	2.88& 26.5\\
\noalign{\smallskip}\hline
\end{tabular}
\end{adjustbox}
\end{table*}

In the next experiment, the time spent on problem-solving for some artificial event logs with different$\lvert A_L \rvert$  and noise levels was analyzed. Using artificial event logs allows us to test the model on logs of different predetermined $\lvert A_L \rvert$  and noise levels. In this experiment, we used $\lvert A_L \rvert$  levels ranging from 10 to 160 which cover a wide range of $\lvert A_L \rvert$  levels that process mining problems usually have. For each $\lvert A_L \rvert$  level, using a ProM 6.6 plugin named “Generate Block-Structured Stochastic Petri Net” 5 different random process models were created. Then, for each of the process models, the Gena \cite {Ref29} plugin for ProM 6.6 was used to create the artificial event log having different noise levels (5\%, 10\%, 15\%, and 20\%). All event logs were created with 10000 traces. 

\begin{figure}
 \begin{center}
  \includegraphics [width=0.5\textwidth]{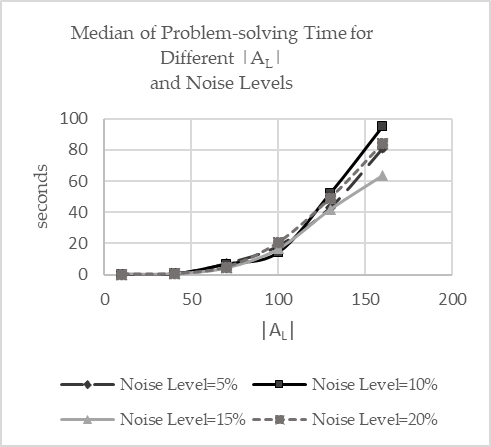}
 \end{center}
\caption{Median of problem-solving time for different $\lvert A_L \rvert$   and noise levels}
\label{fig:2}       
\end{figure}

Using all configurations mentioned in Table ~\ref{tab:3}, the proposed ILP model was applied to each obtained synthetic event log. For each $\lvert A_L \rvert$   and each noise level, the median of achieved problem-solving time is depicted in Fig. ~\ref{fig:2}. According to the results, the noise level had not an important effect on problem-solving time, whereas, $\lvert A_L \rvert$   was more effective. Nevertheless, even in high levels of  $\lvert A_L \rvert$ , (i. e. 130 and 160) the median of time spent on problem-solving was under 100 seconds which is considered as an acceptable time for solving process mining problems. In low levels of  $\lvert A_L \rvert$ , the problem was solved in a very short time. Table~\ref{tab:9} presents the highest time spent on problem-solving for different $\lvert A_L \rvert$   levels. It can be seen that only for $\lvert A_L \rvert$  =160 the maximum spent time was higher than 1000 seconds. In the other levels, the maximum spent time was under 251 seconds which is an acceptable problem-solving time.

\begin{table}
\begin{center}
\caption{Time spent on problem-solving by the proposed method (seconds)}
\label{tab:9}       
\begin{adjustbox}{width=0.6\textwidth}
\begin{tabular}{lcccccccccccccc}
\hline\noalign{\smallskip}
	&	\multicolumn{6}{c}{Event Logs}\\
\noalign{\smallskip}\cline{2-7}\noalign{\smallskip}
		& 10	&	40	&	70	& 100&	130& 160	  \\		
\noalign{\smallskip}\cline{2-7}\noalign{\smallskip}
Maximum\\
 problem-solving\\
 time (seconds)		&		0.98	&	44.34	&	17.52	&	250.21  &119.42& $>$1000 \\

\noalign{\smallskip}\hline
\end{tabular}
\end{adjustbox}
\end{center}
\end{table}
 According to the experiments performed on the real-life and synthetic event logs with a wide range of $\lvert A_L \rvert$  levels (10-160), the median of the solution time was under 100 seconds for high $\lvert A_L \rvert$  levels (i.e., 130 and 160), and under 20 seconds for the other levels. This is supposed as a short (or at least acceptable) time for solving process mining problems. Therefore, for most process mining problems, it is anticipated that the proposed method can be applied without any serious concerns about solution time. 

\section{conclusion and Future Works}\label{sec6}
Dependency graph discovery is an essential step in heuristic-based process discovery, which is a widely used and popular approach in this area. The current heuristics-based methods select the initial set of dependency graph arcs concerning some thresholds for minimum dependency measures and then modify this set. This can result in choosing a non-optimal set of arcs. Also, the modifications can lead to modeling infrequent behaviors (i.e., less precision) and increasing the dependency graph size (i.e., less simplicity). Furthermore, when the dependency graphs extracted by the existing methods contain loops, some tasks are capable of not being on a path from the initial task to the final task. Therefore, this paper introduces a new mathematical programming model, which selects the optimal dependency graph arcs according to the dependency measures of arcs/loops. Simultaneously, it eliminates the flaws mentioned above of the existing methods by using appropriate constraints. It also offers some flexibility in the procedure of extracting the dependency graphs, i.e.; it can restrict the minimum dependency measure of the arcs/loops present in the graph. It also contains options in dealing with short and length two loops. Nevertheless, in this approach, employing many other types of domain knowledge and user-desired flexibility is simply available by introducing relevant constraints. This can be an advantage with high applicability in real-world problems.

The proposed method was assessed using real publicly available event logs. Accordingly, the proposed method outperformed the prominent methods of dependency graph discovery in attaining simple, high-quality dependency graphs. Also, in most cases, among the various dependency graphs extracted by applying different methods to each event log, the highest quality output (considering all three dimensions of fitness, precision, and simplicity) was obtained using the proposed method. It also indicated superiority over the other methods since it is mathematically guaranteed that, even when the dependency graphs extracted by the proposed method contain loops, all graph tasks are on a path from the initial task to the final task. Furthermore, according to the experiments, the proposed method is expected to be solved at a short or at least acceptable time for most sizes of process mining problems.

In future work, more discussions focusing only on the flexible discovery of dependency graphs and the related constraints can be presented. Another possible future research topic can be defining a multi-objective mathematical model in which each objective function corresponds to a quality measure of process models.




\end{document}